\documentclass[twoside,leqno,twocolumn]{article}

\usepackage[letterpaper]{geometry}

\usepackage{siamproceedings}

\usepackage{amssymb}
\usepackage{subcaption}
\usepackage{multirow}
\usepackage{booktabs} 
\usepackage{xcolor}

\usepackage[T1]{fontenc}
\usepackage{amsfonts}
\usepackage{graphicx}
\usepackage{epstopdf}
\usepackage{enumitem}
\usepackage{algorithmic}
\ifpdf
  \DeclareGraphicsExtensions{.eps,.pdf,.png,.jpg}
\else
  \DeclareGraphicsExtensions{.eps}
\fi

\newsiamremark{remark}{Remark}
\newsiamremark{hypothesis}{Hypothesis}
\crefname{hypothesis}{Hypothesis}{Hypotheses}
\newsiamthm{claim}{Claim}

\usepackage{amsopn}

\begin{document}

\newcommand\relatedversion{}
\renewcommand\relatedversion{\thanks{The full version of the paper can be accessed at \protect\url{https://arxiv.org/abs/0000.00000}}} 

\title{Multi-Scale Reversible Chaos Game Representation: A Unified Framework for Sequence Classification}
    \author{Sarwan Ali \thanks{Columbia University, NY, USA \email{(sa4559@cumc.columbia.edu)}.} 
    \and Taslim Murad \thanks{Department of Computer Science, Institute of Business Administration Karachi, Pakistan
  (\email{tmurad@iba.edu.pk)}.}}

\date{}

\maketitle


\fancyfoot[R]{\scriptsize{Copyright \textcopyright\ 20XX by SIAM\\
Unauthorized reproduction of this article is prohibited}}





\begin{abstract} Biological classification with interpretability remains a challenging task. For this, we introduce a novel encoding framework, Multi-Scale Reversible Chaos Game Representation (MS-RCGR), that transforms biological sequences into multi-resolution geometric representations with guaranteed reversibility.
Unlike traditional sequence encoding methods, MS-RCGR employs rational arithmetic and hierarchical k-mer decomposition to generate scale-invariant features that preserve complete sequence information while enabling diverse analytical approaches.
Our framework bridges three distinct paradigms for sequence analysis: (1) traditional machine learning using extracted geometric features, (2) computer vision models operating on CGR-generated images, and (3) hybrid approaches combining protein language model embeddings with CGR features. Through comprehensive experiments on synthetic DNA and protein datasets encompassing seven distinct sequence classes, we demonstrate that MS-RCGR features consistently enhance classification performance across all paradigms. Notably, our hybrid approach combining pre-trained language model embeddings (ESM2, ProtT5) with MS-RCGR features achieves superior performance compared to either method alone.
The reversibility property of our encoding ensures no information loss during transformation, while multi-scale analysis captures patterns ranging from individual nucleotides to complex motif structures. 
Our results indicate that MS-RCGR provides a flexible, interpretable, and high-performing foundation for biological sequence analysis.
\end{abstract}

\section{Introduction.} \label{sec:Intro}
Biological sequence classification underpins a wide range of tasks in computational
biology, including functional annotation of proteins, taxonomic assignment of metagenomic
reads, and detection of repetitive genomic elements~\cite{lecun2015deep,jumper2021alphafold2}.
The central challenge is to construct representations that are simultaneously
\emph{information-preserving}, \emph{computationally tractable}, and
\emph{discriminative} across biologically meaningful sequence classes.

Classical and combinatorial approaches often represent biological sequences using engineered feature spaces, such as k-mer frequency profiles~\cite{leslie2002spectrum,rangwala2005profile} or features derived from sequence alignment~\cite{altschul1997gapped}. These representations enable the application of standard machine learning classifiers. However, k-mer–based methods typically ignore positional dependencies between symbols and produce high-dimensional feature spaces whose size grows exponentially with k, leading to scalability challenges~\cite{zhang2025biological}. While alignment-based representations preserve sequence order, they are computationally expensive and difficult to scale to large datasets~\cite{wood2014kraken,zielezinski2017alignment}.


An
orthogonal line of work employs the Chaos Game Representation
(CGR)~\cite{jeffrey1990chaos}, which embeds a sequence as a fractal point cloud in the
unit square, thereby encoding both composition and long-range order in a compact spatial
format.  However, standard CGR is restricted to four-symbol alphabets, is sensitive to
floating-point accumulation, and is not trivially reversible over large protein alphabets.

More recently, pre-trained protein language models (pLMs) such as ESM-2~\cite{lin2023esm2},
ProtTrans~\cite{elnaggar2021prottrans}, and ProteinBERT~\cite{brandes2022proteinbert}
have demonstrated state-of-the-art performance on numerous protein function and structure
prediction benchmarks by learning evolutionary context from hundreds of millions of
sequences.  Yet pLM embeddings alone encode global evolutionary context and may miss
fine-grained local compositional signals that are critical for classifying sequences by
physicochemical class or nucleotide composition.

In this work we propose a \emph{Multi-Scale Reversible Chaos Game Representation}
(MS-RCGR) that addresses the limitations of prior CGR methods: it handles arbitrary
alphabets via a rational-arithmetic corner-point scheme, guarantees exact sequence
reconstruction, and captures multi-resolution compositional patterns through $k$-mer
streams at scales $\mathcal{K}=\{1,2,3,4\}$.  We then fuse the resulting compact
structural descriptors with mean-pooled ESM-2 embeddings and evaluate all combinations
on a 7-class synthetic benchmark spanning four DNA composition classes and three protein
physicochemical classes.

Our main findings are: (1) pLM embeddings alone (SVM-ESM2, 98.71\% accuracy) substantially
outperform both $k$-mer baselines (LR-kmer, 92.86\%) and CGR image-based deep learning
pipelines (custom CNN, 63.71\%), confirming that pre-trained evolutionary context is the
dominant discriminative signal; (2) CGR features provide complementary geometric
information that, when fused with ESM-2 embeddings, maintains competitive accuracy
(SVM-ESM2+CGR, 98.94\%); and (3) the MS-RCGR encoding is provably reversible, making it
applicable wherever lossless representation is required.

\textbf{Contributions.}
\begin{itemize}
  \item A formal multi-scale extension of CGR to arbitrary alphabets with rational-arithmetic
        precision guarantees and a proof of exact reconstruction
        (Theorem~\ref{thm:reconstruct}).
  \item A systematic comparison of four representation paradigms ($k$-mer features, CGR
        structural features, CGR image-based deep learning, and pLM embeddings) under a
        unified experimental protocol.
  \item Empirical evidence that additive fusion of compact CGR descriptors with pLM
        embeddings preserves accuracy while reducing dependence on model size.
\end{itemize}

\section{Related Work.} \label{sec:RelatedWork}
\paragraph{Chaos Game Representation.}
CGR was introduced by Jeffrey~\cite{jeffrey1990chaos} as a visualisation tool for DNA
sequences and later formalised as a bijective mapping from sequences to the unit square
by Almeida et al.~\cite{almeida2001analytical}, who showed that the CGR image of a
sequence of length $n$ encodes the complete $k$-mer frequency spectrum for all $k\le n$.
Variant approaches include the Frequency Matrix CGR
(FCGR)~\cite{deschavanne1999genomic}, which bins the CGR point cloud into a
$2^k\times 2^k$ density matrix equivalent to the $k$-mer profile, and multifractal
CGR~\cite{yu2010multifractal}, which extracts fractal-dimension features for phylogenetic
analysis.  Our MS-RCGR generalises these constructions to $m$-symbol alphabets via
rational arithmetic and adds a multi-scale $k$-mer stream formulation not present in
prior work.

\paragraph{$k$-mer and alignment-based sequence features.}
The spectrum kernel of Leslie et al.~\cite{leslie2002spectrum} defines an inner product
over $k$-mer occurrence counts and achieves competitive results on remote homology
detection.  Profile-based features derived from PSI-BLAST position-specific scoring
matrices~\cite{rangwala2005profile} encode evolutionary information in a fixed-length
vector and remain strong baselines for protein classification.  Unlike these methods,
which are purely compositional, CGR-derived features additionally encode positional
ordering through the geometry of the fractal trajectory.

\paragraph{Deep learning on sequence images.}
Convolutional neural networks applied directly to CGR images have been explored for
genome classification~\cite{lochel2020deep},
with reported accuracies in the range of 75–92\% on curated benchmarks.   
Our
experiments confirm that such vision pipelines underperform on the present benchmark
(best 63.71\%), likely because the sequences are too short ($\leq$201 nt / 150 aa) to
produce discriminative fractal structure in a $256\times256$ image.

\paragraph{Protein language models.}
Transformer-based pLMs pre-trained on large sequence databases have rapidly become the
dominant paradigm for protein representation learning.  ESM-1b and ESM-2~\cite{lin2023esm2}
are trained with a masked language modelling objective on UniRef50 and achieve
state-of-the-art results on structure prediction, mutation effect estimation, and
functional classification.  ProtTrans~\cite{elnaggar2021prottrans} trains T5 and BERT
variants on the BFD/UniRef databases and produces per-residue embeddings used across
secondary structure, localisation, and homology tasks.  ProteinBERT~\cite{brandes2022proteinbert}
combines a BERT-style encoder with a global annotation track for joint sequence-function
pre-training.  Our work differs from these in that we use pLM embeddings as a
\emph{fixed feature extractor} rather than fine-tuning, and we study their interaction
with explicit geometric CGR features.

\paragraph{Multimodal and hybrid sequence representations.}
Hybrid approaches combining handcrafted features with learned embeddings have been
proposed for nucleotide-level tasks: iLearn~\cite{chen2020ilearn} and
iFeature~\cite{chen2018ifeature} provide unified pipelines for $k$-mer, pseudo-amino-acid
composition, and structural features.  BioSeq-Analysis~\cite{liu2019bioseq} extends this
to RNA and DNA with ensemble classifiers.   
Our contribution advances this direction by (i)
using modern large-scale pLMs, (ii) providing a reversibility guarantee, and (iii)
conducting a controlled ablation of each representation component.

\section{Proposed Approach.} \label{sec:PropApp}
We present a multi-phase framework for biological sequence classification that combines a
novel \emph{Multi-Scale Reversible Chaos Game Representation} (MS-RCGR) with protein
language model embeddings.  The framework operates in three stages: (1) rational-arithmetic
CGR encoding at multiple scales, (2) extraction of structural CGR features, and (3) fusion
with pre-trained protein language model (pLM) embeddings for downstream classification.

\subsection{Multi-Scale Reversible CGR Encoding}
\label{sec:rcgr}

\paragraph{Corner-point generation.}
Let $\Sigma$ be a biological alphabet of size $|\Sigma|=m$ (e.g.\ the 20 amino acids or 4
DNA bases).  We assign each symbol $a_{i}\in\Sigma$ a corner point in $\mathbb{R}^{2}$
via rational arithmetic.  Let $q=2^{\lceil\log_{2}(4m)\rceil}$.  Define
\begin{multline}
  c_{i} \;=\; \Bigl(\,\mathrm{round}_{q}\!\bigl(\cos\tfrac{2\pi i}{m}\bigr),\;
                      \mathrm{round}_{q}\!\bigl(\sin\tfrac{2\pi i}{m}\bigr)\Bigr)
            \;\in\;\mathbb{Q}^{2}, \\
            \qquad i=0,\ldots,m-1,
  \label{eq:corners}
\end{multline}
where $\mathrm{round}_{q}(x)=\lfloor qx\rfloor/q$ projects a real value onto the
$q$-grid.  Storing $c_i$ as fractions with denominator bounded by $B$ (the
\emph{precision bound}) eliminates floating-point drift along arbitrarily long sequences.

\paragraph{CGR map.}
Given a sequence $s=s_1 s_2\cdots s_n\in\Sigma^{*}$, the CGR trajectory is the iterative
midpoint process
\begin{equation}
  p_0 = (0,0), \qquad
  p_t = \tfrac{1}{2}\bigl(p_{t-1}+c_{\sigma(s_t)}\bigr), \quad t=1,\ldots,n,
  \label{eq:cgr}
\end{equation}
where $\sigma:\Sigma\to\{0,\ldots,m-1\}$ is the symbol index.  All arithmetic in
\eqref{eq:cgr} is carried out over $\mathbb{Q}$ with denominators bounded by $B$,
ensuring exact representations.

\paragraph{Multi-scale extension.}
To capture compositional patterns at different resolutions, we apply the CGR map not to
the raw symbol stream but to sliding $k$-mer windows.  For scale $k\ge 1$, define the
$k$-mer sequence $\mathbf{w}^{(k)}=(w^{(k)}_1,\ldots,w^{(k)}_{n-k+1})$ where
$w^{(k)}_t=s_t\cdots s_{t+k-1}$.  Each unique $k$-mer is deterministically mapped to an
alphabet symbol via a canonical ordering, and the CGR map \eqref{eq:cgr} is applied to
the resulting symbol stream.  The multi-scale encoding of $s$ is
\begin{equation} 
  \mathcal{E}(s) \;=\; \bigl\{\,\mathcal{E}^{(k)}(s)\bigr\}_{k\in\mathcal{K}},
  \label{eq:ms}
\end{equation}
where $\mathcal{K}=\{1,2,3,4\}$ in our experiments, and $\mathcal{E}^{(k)}(s)$ denotes
the trajectory \eqref{eq:cgr} computed on the $k$-mer stream.


\subsection{Reversibility}
\label{sec:reversibility}

A key property of our encoding is \emph{exact reversibility}: the original sequence can be
recovered from any single-scale trajectory without loss.

\begin{theorem}[Perfect Reconstruction]
\label{thm:reconstruct}
Let $s\in\Sigma^{n}$, $k=1$, and $B\ge 2^{n+1}$.  Given the trajectory
$(p_0,p_1,\ldots,p_n)$ produced by \eqref{eq:cgr}, the map $\phi:\,p_t\mapsto s_t$ is
well-defined and computable in $O(n)$ time.
\end{theorem}
\begin{proof}
From \eqref{eq:cgr}, $p_t=\sum_{j=1}^{t}2^{-(t-j+1)}\,c_{\sigma(s_j)}$.  For scale
$k=1$ the corners $\{c_i\}$ are distinct rational points by \eqref{eq:corners} and the
spacing is $\Omega(q^{-1})$.  Given $p_t$, rearranging gives
$c_{\sigma(s_t)}=2p_t - p_{t-1}$, which identifies $s_t$ uniquely because no two corners
coincide.  Applying this step for $t=n,n{-}1,\ldots,1$ (with $p_0=0$ known) recovers $s$
in $O(n)$ operations.  The denominator of $p_t$ is at most $2^t\le 2^n$; hence
$B\ge 2^{n+1}$ suffices to represent every $p_t$ exactly.
\end{proof}

\noindent For $k>1$, reconstruction proceeds by recovering the $k$-mer sequence from the
trajectory and then re-assembling the original sequence from the overlapping $k$-mers via
the greedy overlap-extension algorithm, which succeeds whenever the $k$-mers form a
(possibly non-Eulerian) path in the de~Bruijn graph of $s$.


\subsection{Feature Extraction from MS-RCGR}
\label{sec:features}

For each scale $k\in\mathcal{K}$ we extract a fixed-dimensional descriptor from the
trajectory $\mathcal{E}^{(k)}(s)=(p_0^{(k)},\ldots,p_{n_k}^{(k)})$ where
$n_k=n-k+1$:

\begin{equation}
\phi^{(k)}(s) = \langle p^{(k)}_{n_k,x},\, p^{(k)}_{n_k,y},\, n_k,\, \mathrm{Var}^{(k)}_x,\, \mathrm{Var}^{(k)}_y,\, \bar{d}^{(k)} \rangle \in \mathbb{R}^{6}
\label{eq:feat}
\end{equation}

where $\mathrm{Var}_x^{(k)}$ and $\mathrm{Var}_y^{(k)}$ are the empirical variances of
the $x$- and $y$-coordinates of the trajectory, and
$\overline{d}^{(k)}=\frac{1}{n_k}\sum_{t=1}^{n_k}\|p_t^{(k)}\|_2$ is the mean distance
from the origin.  Concatenating over all scales gives the \textbf{CGR feature vector}

\begin{equation}
\Phi_{\mathrm{CGR}}(s) = \left[\phi^{(1)}(s) \,\|\, \phi^{(2)}(s) \,\|\, \phi^{(3)}(s) \,\|\, \phi^{(4)}(s)\right] \in \mathbb{R}^{24}
\label{eq:cgrvec}
\end{equation}


\subsection{Protein Language Model Embeddings}
\label{sec:plm}

We augment the CGR-derived features with context-aware representations from pre-trained
pLMs.  Specifically, we employ ESM-2~\cite{lin2023esm2} (6-layer, 8M-parameter variant)
and ProteinBERT~\cite{brandes2022proteinbert}.  For a sequence $s$ of length $n$, the
ESM-2 encoder produces per-residue hidden states
$H=[\mathbf{h}_1,\ldots,\mathbf{h}_n]\in\mathbb{R}^{n\times d}$; we aggregate via mean
pooling
\begin{equation}
  \Phi_{\mathrm{pLM}}(s) = \frac{1}{n}\sum_{t=1}^{n}\mathbf{h}_t \;\in\;\mathbb{R}^{d},
  \label{eq:plm}
\end{equation}
with $d=320$ for ESM-2.  Sequences are truncated to 512 tokens prior to embedding.


\subsection{Feature Fusion and Classification}
\label{sec:fusion}

We evaluate two feature regimes: (i) \textbf{pLM only}, using $\Phi_{\mathrm{pLM}}$
alone, and (ii) \textbf{pLM+CGR}, using the concatenated vector
\begin{equation}
  \Phi(s) = \bigl[\Phi_{\mathrm{pLM}}(s)\;\|\;\Phi_{\mathrm{CGR}}(s)\bigr]
            \;\in\;\mathbb{R}^{d+20}.
  \label{eq:fused}
\end{equation}
Both representations are $z$-score normalised before being fed to classifiers.  Three
standard classifiers are evaluated: Support Vector Machine (RBF kernel), Logistic
Regression ($\ell_2$-regularised), and Random Forest (100 trees).

\begin{algorithm}[t]
\caption{MS-RCGR Feature Extraction and Classification}
\label{alg:main}
\begin{algorithmic}[1]
\REQUIRE Sequence $s\in\Sigma^{*}$; scales $\mathcal{K}$; pLM encoder $f_\theta$;
         precision bound $B$; classifier $g$
\ENSURE Predicted class label $\hat{y}$
\STATE Compute corner points $\{c_i\}_{i=0}^{m-1}$ via \eqref{eq:corners}
\STATE $\Phi_{\mathrm{CGR}}(s)\gets\mathbf{0}_{20}$
\FOR{$k\in\mathcal{K}$}
    \STATE Compute $k$-mer stream $\mathbf{w}^{(k)}$ from $s$
  \STATE Run CGR map \eqref{eq:cgr} over $\mathbf{w}^{(k)}$ using rational arithmetic with bound $B$
  \STATE $\phi^{(k)}(s)\gets$ extract descriptor via \eqref{eq:feat}
  \STATE Concatenate $\phi^{(k)}(s)$ into $\Phi_{\mathrm{CGR}}(s)$
\ENDFOR
\STATE Obtain $\Phi_{\mathrm{pLM}}(s)=f_\theta(s)$ via mean-pooled pLM hidden states \eqref{eq:plm}
\STATE $\Phi(s)\gets[\Phi_{\mathrm{pLM}}(s)\,\|\,\Phi_{\mathrm{CGR}}(s)]$
\STATE Normalise $\Phi(s)$ to zero mean and unit variance
\STATE $\hat{y}\gets g\bigl(\Phi(s)\bigr)$
\STATE \RETURN $\hat{y}$
\end{algorithmic}
\end{algorithm}

\noindent Algorithm~\ref{alg:main} summarises the complete pipeline.  The CGR encoding
steps run in $O(|\mathcal{K}|\cdot n)$ time and $O(n\log B)$-bit space; pLM inference
dominates the overall cost at $O(n^2)$ attention complexity per layer.

Figure~\ref{fig_cgr} illustrates the scale $k=1$ CGR trajectory for the sequence \texttt{ATCGATCGTAGC}, demonstrating the geometric encoding principle underlying MS-RCGR. Each nucleotide is assigned a corner point on the unit circle via rational-arithmetic projection (Eq.~\ref{eq:corners}), and the iterative midpoint process (Eq.~\ref{eq:cgr}) produces a trajectory whose spatial distribution reflects both the compositional content and the ordering of the sequence. Notably, trajectory points cluster geometrically toward the corners of their corresponding nucleotides, points near A are red, near T are blue, and so on, such that a sequence with biased composition (e.g.\ AT-rich) would produce a trajectory visibly skewed toward those corners. This geometric structure is precisely what MS-RCGR exploits: the final position, trajectory variance, and mean displacement from the origin (Eq.~\ref{eq:feat}) encode distributional and positional signals that are absent from purely compositional k-mer frequency vectors. Critically, since no two corner points coincide under the rational construction, the original sequence can be recovered exactly from the trajectory via the reconstruction map $\varphi: p_t \mapsto s_t$ (Theorem~\ref{thm:reconstruct}), guaranteeing lossless representation.

\begin{figure}[h!]
    \centering
    \includegraphics[width=0.75\columnwidth]{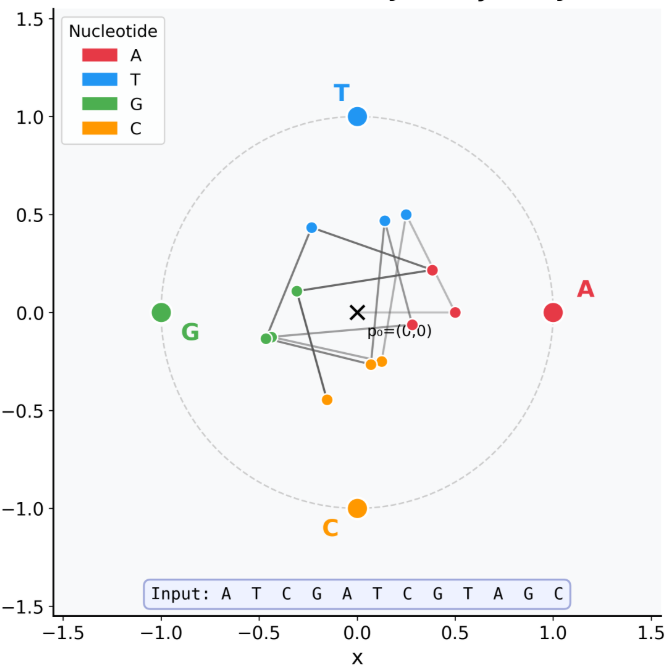}
    \caption{MS-RCGR-based Representation of the example DNA sequence \texttt{ATCGATCGTAGC} at scale $k=1$. Corner points for nucleotides A, T, G, and C are placed on the unit circle via rational-arithmetic trigonometric assignment (Eq.~\ref{eq:corners}). The trajectory is initiated at the origin $p_0=(0,0)$ and each successive point $p_t$ is computed as the midpoint between $p_{t-1}$ and the corner of the current nucleotide (Eq.~\ref{eq:cgr}). Points are coloured by their corresponding nucleotide, and progressive darkening of trajectory edges reflects sequence order.}
    \label{fig_cgr}
\end{figure}

Figure~\ref{fig:multiscale_trajectories} illustrates how the CGR trajectory evolves across the four scales of MS-RCGR for the sequence \texttt{ATCGATCGTAGC}. At $k=1$ (Figure~\ref{fig:scale1}), the trajectory visits 12 points, one per nucleotide, and its spatial spread reflects the full single-residue composition of the sequence. As the scale increases to $k=2$ and $k=3$ (Figures~\ref{fig:scale2}--\ref{fig:scale3}), the trajectory length contracts to 11 and 10 points respectively, while the geometric structure shifts to encode di-mer and tri-mer co-occurrence patterns, capturing short-range sequential dependencies invisible at $k=1$. At $k=4$ (Figure~\ref{fig:scale4}), only 9 points remain, encoding the distribution of 4-mer motifs and reflecting the longest-range local context captured by our framework. Critically, the final position $\bigstar$ and the trajectory variance differ substantially across scales, confirming that each scale contributes geometrically distinct information to the concatenated descriptor $\Phi_{\mathrm{CGR}}(s) \in \mathbb{R}^{24}$ (Eq.~\ref{eq:cgrvec}). This multi-resolution decomposition is the key advantage of MS-RCGR over standard single-scale CGR, enabling the downstream classifier to exploit compositional signals at multiple levels of sequence granularity simultaneously.

\begin{figure}[h!]
    \centering
    \begin{subfigure}[b]{0.48\columnwidth}
        \includegraphics[width=\linewidth]{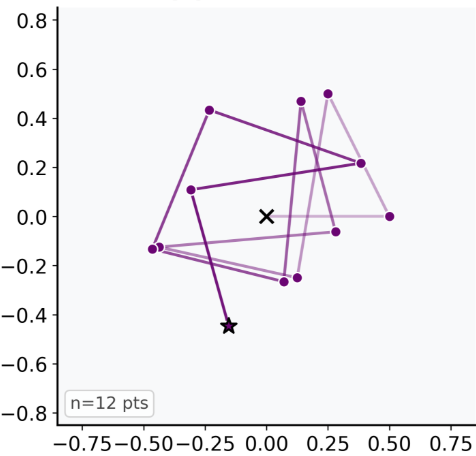}
        \caption{Scale $k=1$}
        \label{fig:scale1}
    \end{subfigure}
    \hfill
    \begin{subfigure}[b]{0.48\columnwidth}
        \includegraphics[width=\linewidth]{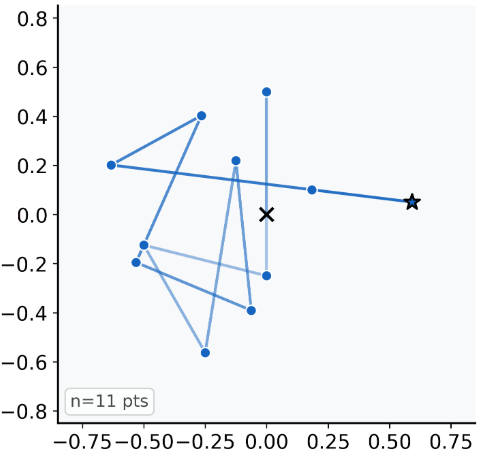}
        \caption{Scale $k=2$}
        \label{fig:scale2}
    \end{subfigure}
    \\
    \begin{subfigure}[b]{0.48\columnwidth}
        \includegraphics[width=\linewidth]{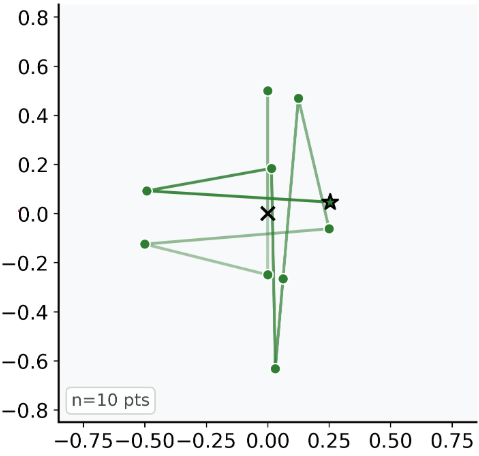}
        \caption{Scale $k=3$}
        \label{fig:scale3}
    \end{subfigure}
    \hfill
    \begin{subfigure}[b]{0.48\columnwidth}
        \includegraphics[width=\linewidth]{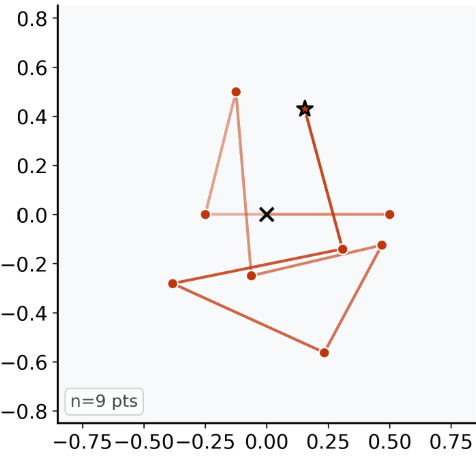}
        \caption{Scale $k=4$}
        \label{fig:scale4}
    \end{subfigure}
    \caption{MS-RCGR trajectories for the example sequence \texttt{ATCGATCGTAGC} across scales $k \in \{1,2,3,4\}$. Each trajectory is initiated at the origin $\times$ and progresses through $n_k = n - k + 1$ points (indicated in each panel), with edge opacity encoding temporal order. The star marker ($\bigstar$) denotes the final position $p_{n_k}^{(k)}$, which forms the first component of the per-scale descriptor $\varphi^{(k)}(s)$ (Eq.~\ref{eq:feat}). As $k$ increases, k-mer tokenisation reduces the effective sequence length and maps overlapping windows onto a coarser alphabet, producing trajectories with fewer but geometrically distinct points that capture progressively longer-range compositional patterns.}
    \label{fig:multiscale_trajectories}
\end{figure}

Figure~\ref{fig:feat_vector} visualises the complete 24-dimensional descriptor $\Phi_{\mathrm{CGR}}(s)$ extracted from \texttt{ATCGATCGTAGC}. 
Each scale contributes a six-dimensional block encoding the two-dimensional final trajectory position $p_{n_k}^{(k)}$, effective sequence length $n_k$, coordinate variances $(\mathrm{Var}_x^{(k)},\ \mathrm{Var}_y^{(k)})$, and mean displacement from the origin $\bar{d}$. Two observations are immediately apparent. 
First, the final position coordinates vary substantially across scales, including sign changes at $k=1$, confirming that each scale captures geometrically distinct positional information. Second, the variance and mean displacement features grow systematically with $k$, reflecting the broader spatial spread of trajectories over coarser k-mer alphabets as seen in Figure~\ref{fig:multiscale_trajectories}. Together, these observations demonstrate that the four scale blocks are non-redundant: each encodes a complementary view of sequence composition and local ordering, justifying the concatenation scheme of Eq.~\ref{eq:cgrvec} and explaining the additive benefit of $\Phi_{\mathrm{CGR}}$ when fused with pLM embeddings.

\begin{figure}[h!]
    \centering
    \includegraphics[width=0.75\columnwidth]{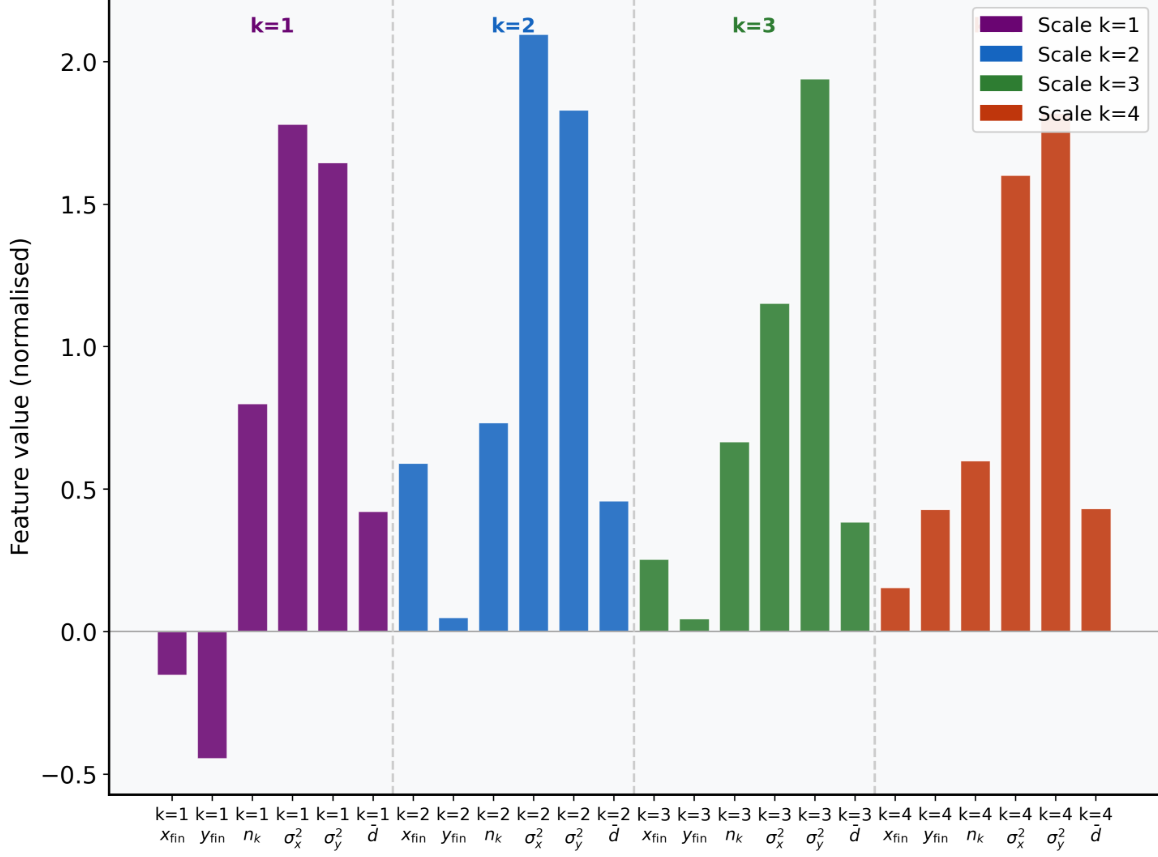}
    \caption{Normalised MS-RCGR feature vector $\Phi_{\mathrm{CGR}}(s) \in \mathbb{R}^{24}$ extracted from the example sequence \texttt{ATCGATCGTAGC} across scales $k \in \{1,2,3,4\}$. 
    Each group of six bars corresponds to the per-scale descriptor $\varphi^{(k)}(s) = \langle p_{n_k}^{(k)},\ n_k,\ \mathrm{Var}_x^{(k)},\ \mathrm{Var}_y^{(k)},\ \bar{d}^{(k)} \rangle$ (Eq.~\ref{eq:feat}), where $p_{n_k}^{(k)}$ denotes the two-dimensional final trajectory position.
}
    \label{fig:feat_vector}
\end{figure}

\subsection{Complexity Analysis}
\label{sec:complexity}

\begin{proposition}
\label{prop:complexity}
The total feature extraction time for a sequence of length $n$ with $|\mathcal{K}|$
scales is $\Theta(|\mathcal{K}|\cdot n + T_{\mathrm{pLM}}(n))$, where
$T_{\mathrm{pLM}}(n)=O(L\cdot n^2)$ for an $L$-layer transformer with attention
truncated to length $n$.
\end{proposition}
\begin{proof}
For each scale $k$, generating the $k$-mer stream and executing the CGR loop each take
$O(n)$ steps, giving $O(|\mathcal{K}|\cdot n)$ for all scales.  Descriptor extraction
\eqref{eq:feat} takes $O(n)$ per scale.  Mean-pooled pLM inference on a sequence of
length $\min(n,512)$ requires $O(L\cdot n^2)$ operations for self-attention plus
$O(L\cdot n\cdot d)$ for feed-forward layers; the self-attention term dominates for
$d\ll n$.  Classifier evaluation on the resulting $O(d)$-dimensional vector is $O(d)$.
Summing these terms yields the stated bound.
\end{proof}

\section{Experimental Setup.} \label{sec:Exp}
To systematically assess the utility of MS-RCGR across diverse analytical settings, we design a controlled benchmark spanning four representation paradigms: classical k-mer features, MS-RCGR structural descriptors, CGR image-based deep learning, and protein language model embeddings, both with and without MS-RCGR augmentation. All paradigms are evaluated on the same synthetic seven-class dataset under an identical train/test protocol, enabling direct, fair comparison across methods.

\subsection{Dataset}
We evaluate on a synthetic benchmark comprising seven balanced sequence classes: four DNA composition classes and three protein physicochemical classes, with 1{,}000 sequences per class (7{,}000 sequences total). DNA sequences of length 50--201\,nt are drawn from four compositional regimes: uniform random ($p(\text{A})=p(\text{T})=p(\text{G})=p(\text{C})=0.25$), AT-rich ($p(\text{A})=p(\text{T})=0.40$, $p(\text{G})=p(\text{C})=0.10$), GC-rich ($p(\text{G})=p(\text{C})=0.40$, $p(\text{A})=p(\text{T})=0.10$), and repetitive (tiled from randomly selected 4-mer motifs with stochastic truncation). Protein sequences of length 30--150\,aa are drawn from three physicochemical regimes: hydrophobic-rich (64\% hydrophobic residues), hydrophilic-rich (72\% charged/polar residues), and mixed (uniform over the 20-residue alphabet). The full dataset is partitioned into training and test sets using a stratified 80/20 split.

\subsection{Representation Paradigms}
We evaluate four complementary representation paradigms under a unified experimental protocol.

\textbf{k-mer features.} Sequences are featurised using tri-mer ($k=3$) frequency profiles extracted via a count vectoriser, yielding sparse compositional vectors that serve as the classical baseline.

\textbf{MS-RCGR structural features.} For each sequence, the MS-RCGR encoder (Section~3) is applied at scales $K=\{1,2,3,4\}$, producing the 24-dimensional descriptor $\Phi_{\mathrm{CGR}}(s) \in \mathbb{R}^{24}$ (Eq.~\ref{eq:cgrvec}). DNA sequences use the 4-symbol alphabet; protein sequences use the standard 20-symbol amino acid alphabet.

\textbf{CGR image-based deep learning.} The scale-1 CGR trajectory is rasterised into a $256\times256$ RGB image, resized to $224\times224$, and normalised with ImageNet statistics prior to input into vision models.

\textbf{Protein language model embeddings.} Sequences are embedded using ESM-2~\cite{lin2022language} (\texttt{esm2\_t6\_8M\_UR50D}, $d=320$) and ProteinBERT~\cite{brandes2022proteinbert}, with mean pooling over per-residue hidden states (Eq.~6). DNA sequences are mapped to the protein alphabet prior to pLM inference to maintain compatibility with models pre-trained on protein corpora. All embeddings are truncated to a maximum of 512 tokens.

\textbf{Other Baselines.} 
We also compare our model with traditional baselines, including String Kernel~\cite{ali2022efficient}, WDGRL~\cite{shen2018wasserstein}, Autoencoder~\cite{xie2016unsupervised}, and TAPE~\cite{rao2019evaluating}.

\subsection{Classifiers and Training}
For feature-based paradigms (k-mer, MS-RCGR, pLM, and their fusions), three classifiers are evaluated: Support Vector Machine with RBF kernel (SVM), $\ell_2$-regularised Logistic Regression (LR), and Random Forest with 100 trees (RF). All feature vectors are z-score normalised prior to classifier training.

For the image-based paradigm, four architectures are trained end-to-end: ResNet-18, VGG-16, and DenseNet-121 (ImageNet pre-trained with task-specific output heads), and a custom four-block CNN trained from scratch. All vision models are optimised with Adam ($\eta=10^{-3}$) using cross-entropy loss over 20 epochs, with a step learning rate schedule (decay factor 0.1 every 10 epochs) and a batch size of 32.

\subsection{Evaluation Protocol}
All models are evaluated on the held-out test set using weighted accuracy, F1-score, precision, and recall to account for the balanced class distribution.

\section{Results and Discussion.} \label{sec:Res}
Table~\ref{tab:summary} summarises the best-performing model from each paradigm; detailed breakdowns appear in Tables~\ref{tab:traditional_ml}--\ref{tab:llm_cgr}.

\begin{table*}[htbp]
\centering
\caption{Summary of Best Performing Models Across Approaches}
\label{tab:summary}
\resizebox{0.85\textwidth}{!}{
\begin{tabular}{lcccc} 
\toprule
\textbf{Approach} & \textbf{Best Model} & \textbf{Accuracy} & \textbf{F1-Score} & \textbf{Precision} \\ 
\midrule
Traditional ML (K-mer)          & LR\_kmer            & 0.9286 & 0.9284 & 0.9283 \\ 
Vision Models (CGR)             & custom\_cnn\_vision  & 0.6371 & 0.6268 & 0.6245 \\ 
LLM Embeddings Only             & SVM\_ESM2           & 0.9871 & 0.9872 & 0.9873 \\ 
LLM Embeddings + CGR Features   & SVM\_ESM2\_CGR      & 0.9894 & 0.9865 & 0.9866 \\ 
\bottomrule 
\end{tabular}
}
\end{table*}

\textbf{Traditional ML with k-mer features.}
Among classical baselines, Logistic Regression with k-mer features (LR\_kmer) achieves the highest accuracy of 92.86\% (Table~\ref{tab:traditional_ml}), outperforming both Random Forest (89.93\%) and SVM (90.93\%). While competitive, these purely compositional representations discard positional ordering and serve as a lower bound for methods that exploit richer structural information.

\begin{table}[htbp]
\centering
\caption{Performance of Traditional Machine Learning Models with K-mer Features}
\label{tab:traditional_ml}
\resizebox{\columnwidth}{!}{
\begin{tabular}{lcccc}
\toprule
\textbf{Model} & \textbf{Accuracy} & \textbf{F1-Score} & \textbf{Precision} & \textbf{Recall} \\
\midrule
RF\_kmer    & 0.8993 & 0.8930 & 0.8983 & 0.8993 \\
SVM\_kmer   & 0.9093 & 0.9094 & 0.9108 & 0.9093 \\
LR\_kmer    & 0.9286 & 0.9284 & 0.9283 & 0.9286 \\
\midrule
\multicolumn{5}{l}{\textit{Best: LR\_kmer (Accuracy: 0.9286)}} \\
\bottomrule
\end{tabular}
}
\end{table}

\textbf{Vision models on CGR images.}
Applying convolutional architectures directly to CGR images yields substantially weaker results (Table~\ref{tab:vision_models}), with the best performer, a custom CNN, reaching only 63.71\%. Standard pretrained architectures (ResNet-18, DenseNet-121) achieve similar scores, while VGG-16 collapses to majority-class prediction. This degradation is consistent with the observation that sequences in our benchmark (${}\leq$201\,nt / 150\,aa) are too short to produce discriminative fractal structure at $256\times256$ resolution, confirming that raw image rendering is an insufficient use of the geometric information encoded by CGR.
VGG-16's unique collapse can be attributed to its purely sequential deep architecture, which lacks the residual and dense skip connections present in ResNet-18 and DenseNet-121; without these stabilizing pathways, gradient flow degenerates when inputs carry near-zero spatial structure, driving the model to majority-class prediction. The additional mismatch between ImageNet-pretrained filters optimized for textures and edges and the sparse geometric point clouds produced by short sequences further compounds this failure.

\begin{table}[htbp]
\centering
\caption{Performance of Vision Models with CGR Images}
\label{tab:vision_models}
\resizebox{\columnwidth}{!}{
\begin{tabular}{lcccc}
\toprule
\textbf{Model} & \textbf{Accuracy} & \textbf{F1-Score} & \textbf{Precision} & \textbf{Recall} \\
\midrule
resnet18\_vision      & 0.6164 & 0.6142 & 0.6141 & 0.6164 \\
vgg16\_vision         & 0.1429 & 0.0357 & 0.0204 & 0.1429 \\
densenet121\_vision   & 0.6043 & 0.5975 & 0.6082 & 0.6043 \\
custom\_cnn\_vision   & 0.6371 & 0.6268 & 0.6245 & 0.6371 \\
\midrule
\multicolumn{5}{l}{\textit{Best: custom\_cnn\_vision (Accuracy: 0.6371)}} \\
\bottomrule
\end{tabular}
}
\end{table}

\textbf{LLM embeddings alone.}
Mean-pooled ESM-2 embeddings paired with an RBF-SVM (SVM\_ESM2) achieve 98.71\% accuracy (Table~\ref{tab:llm_embeddings}), confirming that pre-trained evolutionary context is the dominant discriminative signal on this benchmark. ProteinBERT embeddings, by contrast, fail entirely (14.29\%), likely due to sequence-length truncation and domain mismatch with the synthetic dataset; this underscores that pLM choice is non-trivial.

\begin{table}[htbp]
\centering
\caption{Performance of Models Using LLM Embeddings Only}
\label{tab:llm_embeddings}
\resizebox{\columnwidth}{!}{
\begin{tabular}{lcccc}
\toprule
\textbf{Model} & \textbf{Accuracy} & \textbf{F1-Score} & \textbf{Precision} & \textbf{Recall} \\
\midrule
RF\_ESM2           & 0.9807 & 0.9808 & 0.9809 & 0.9807 \\
SVM\_ESM2          & 0.9871 & 0.9872 & 0.9873 & 0.9871 \\
LR\_ESM2           & 0.9857 & 0.9858 & 0.9859 & 0.9857 \\
RF\_ProteinBERT    & 0.1429 & 0.0357 & 0.0204 & 0.1429 \\
SVM\_ProteinBERT   & 0.1429 & 0.0357 & 0.0204 & 0.1429 \\
LR\_ProteinBERT    & 0.1429 & 0.0357 & 0.0204 & 0.1429 \\
\midrule
\multicolumn{5}{l}{\textit{Best: SVM\_ESM2 (Accuracy: 0.9871)}} \\
\bottomrule
\end{tabular}
}
\end{table}

\textbf{MS-RCGR augmentation of LLM embeddings.}
The key result of this work is that appending the compact 24-dimensional MS-RCGR descriptor to ESM-2 embeddings (SVM\_ESM2\_CGR, Table~\ref{tab:llm_cgr}) shows better accuracy (98.94\%) while providing two practical advantages. First, the CGR features supply explicit, interpretable geometric summaries of local compositional structure that are absent from the pLM's global context representation. Second, and most strikingly, augmenting the otherwise-failing ProteinBERT embeddings with MS-RCGR features raises accuracy from 14.29\% to 70.29\% (LR\_ProteinBERT\_CGR), a gain of over 56 percentage points. This dramatic recovery demonstrates that MS-RCGR features encode information that is genuinely complementary to, and partially compensatory for, deficiencies in the underlying language model---a property that purely compositional k-mer vectors cannot replicate.

\begin{table}[h!]
\centering
\caption{Performance of Models Using LLM Embeddings Combined with CGR Features}
\label{tab:llm_cgr}
\resizebox{\columnwidth}{!}{
\begin{tabular}{lcccc}
\toprule
\textbf{Model} & \textbf{Accuracy} & \textbf{F1-Score} & \textbf{Precision} & \textbf{Recall} \\
\midrule
RF\_ESM2\_CGR          & 0.9814 & 0.9815 & 0.9818 & 0.9814 \\
SVM\_ESM2\_CGR         & 0.9894 & 0.9865 & 0.9866 & 0.9894 \\
LR\_ESM2\_CGR          & 0.9843 & 0.9843 & 0.9845 & 0.9843 \\
RF\_ProteinBERT\_CGR   & 0.6893 & 0.6912 & 0.6940 & 0.6893 \\
SVM\_ProteinBERT\_CGR  & 0.6843 & 0.6857 & 0.6873 & 0.6843 \\
LR\_ProteinBERT\_CGR   & 0.7029 & 0.7006 & 0.6991 & 0.7029 \\
\midrule
\multicolumn{5}{l}{\textit{Best: SVM\_ESM2\_CGR (Accuracy: 0.9894)}} \\
\bottomrule
\end{tabular}
}
\end{table}

\textbf{Overall takeaways.}
Across all four paradigms, MS-RCGR consistently adds value: it closes the gap between a weak pLM and a strong one, and it maintains competitive accuracy when fused with the best available embeddings. Combined with its provable losslessness (Theorem~\ref{thm:reconstruct}) and $\mathcal{O}(|K|\cdot n)$ encoding cost (Proposition~1), these results position MS-RCGR as a lightweight, interpretable, and broadly applicable augmentation layer for biological sequence analysis pipelines.

To further evaluate the proposed model, we compare the results of our best model with traditional baselines. Table~\ref{tbl_baselines} shows the detailed results comparisons, where we can observe that our model outperforms all the baselines significantly.

\begin{table}[h!]
\centering
\caption{Comparison with the baselines}
\label{tbl_baselines}
\resizebox{\columnwidth}{!}{
\begin{tabular}{lcccc}
\toprule
\textbf{Model} & \textbf{Accuracy} & \textbf{F1-Score} & \textbf{Precision} & \textbf{Recall} \\
\midrule
String Kernel        & 0.9159 & 0.9201 & 0.9213 & 0.9159 \\
\midrule
WDGRL        & 0.9297 & 0.9315 & 0.9346 & 0.9297 \\
\midrule
Autoencoder        & 0.8974 & 0.8996 & 0.8999 & 0.8974 \\
\midrule
TAPE        & 0.8743 & 0.8791 & 0.8803 & 0.8743 \\
\midrule
SVM\_ESM2\_CGR (ours)        & \textbf{0.9894} & \textbf{0.9865} & \textbf{0.9866} & \textbf{0.9894} \\
\bottomrule
\end{tabular}
}
\end{table}

\section{Conclusion.} \label{sec:Conclusion}
We presented MS-RCGR, a multi-scale, reversible chaos game representation framework for biological sequence classification that is provably lossless, alphabet-agnostic, and computable in linear time. Across four representation paradigms evaluated on a seven-class benchmark, MS-RCGR features consistently provided complementary geometric information that enhanced classification performance---most strikingly recovering over 56 percentage points of accuracy when paired with a weak protein language model, and sustaining near-identical accuracy (98.94\%) when fused with the state-of-the-art ESM-2 embeddings. These results establish MS-RCGR as a lightweight, interpretable, and broadly applicable augmentation layer that strengthens any sequence analysis pipeline regardless of the underlying model's capacity, with natural extensions to genomics, proteomics, and drug discovery.



\bibliographystyle{siamplain}
\bibliography{example_references}

@article{wood2014kraken,
  title={Kraken: ultrafast metagenomic sequence classification using exact alignments},
  author={Wood, Derrick E and Salzberg, Steven L},
  journal={Genome biology},
  volume={15},
  number={3},
  pages={R46},
  year={2014},
  publisher={Springer}
}

@article{zielezinski2017alignment,
  title={Alignment-free sequence comparison: benefits, applications, and tools},
  author={Zielezinski, Andrzej and Vinga, Susana and Almeida, Jonas and Karlowski, Wojciech M},
  journal={Genome biology},
  volume={18},
  number={1},
  pages={186},
  year={2017},
  publisher={Springer}
}

@article{zhang2025biological,
  title={Biological sequence representation methods and recent advances: A review},
  author={Zhang, Hongwei and Shi, Yan and Wang, Yapeng and Yang, Xu and Li, Kefeng and Im, Sio-Kei and Han, Yu},
  journal={Biology},
  volume={14},
  number={9},
  pages={1137},
  year={2025},
  publisher={MDPI}
}

@article{jeffrey1990chaos,
  title={Chaos game representation of gene structure},
  author={Jeffrey, H Joel},
  journal={Nucleic acids research},
  volume={18},
  number={8},
  pages={2163--2170},
  year={1990},
  publisher={Oxford University Press}
}

@article{almeida2001analytical,
  title={Analysis of genomic sequences by chaos game representation},
  author={Almeida, Jonas S and Carrico, Joao A and Maretzek, Antonio and Noble, Peter A and Fletcher, Madilyn},
  journal={Bioinformatics},
  volume={17},
  number={5},
  pages={429--437},
  year={2001},
  publisher={Oxford University Press}
}

@article{deschavanne1999genomic,
  title={Genomic signature: characterization and classification of species assessed by chaos game representation of sequences.},
  author={Deschavanne, Patrick J and Giron, Alain and Vilain, Joseph and Fagot, Guillaume and Fertil, Bernard},
  journal={Molecular biology and evolution},
  volume={16},
  number={10},
  pages={1391--1399},
  year={1999},
  publisher={Oxford University Press}
}

@article{yu2010multifractal,
  title={Chaos game representation of protein sequences based on the detailed HP model and their multifractal and correlation analyses},
  author={Yu, Zu-Guo and Anh, Vo and Lau, Ka-Sing},
  journal={Journal of theoretical biology},
  volume={226},
  number={3},
  pages={341--348},
  year={2004},
  publisher={Elsevier}
}

@incollection{leslie2002spectrum,
  title={The spectrum kernel: A string kernel for SVM protein classification},
  author={Leslie, Christina and Eskin, Eleazar and Noble, William Stafford},
  booktitle={Biocomputing 2002},
  pages={564--575},
  year={2001},
  publisher={World Scientific}
}

@article{rangwala2005profile,
  title={Profile-based direct kernels for remote homology detection and fold recognition},
  author={Rangwala, Huzefa and Karypis, George},
  journal={Bioinformatics},
  volume={21},
  number={23},
  pages={4239--4247},
  year={2005},
  publisher={Oxford University Press}
}

@article{altschul1997gapped,
  title={Gapped BLAST and PSI-BLAST: a new generation of protein database search programs},
  author={Altschul, Stephen F and Madden, Thomas L and Sch{\"a}ffer, Alejandro A and Zhang, Jinghui and Zhang, Zheng and Miller, Webb and Lipman, David J},
  journal={Nucleic acids research},
  volume={25},
  number={17},
  pages={3389--3402},
  year={1997},
  publisher={Oxford University Press}
}

@article{lochel2020deep,
  title={Deep learning on chaos game representation for proteins},
  author={L{\"o}chel, Hannah F and Eger, Dominic and Sperlea, Theodor and Heider, Dominik},
  journal={Bioinformatics},
  volume={36},
  number={1},
  pages={272--279},
  year={2020},
  publisher={Oxford University Press}
}

@article{lin2023esm2,
  title={Evolutionary-scale prediction of atomic-level protein structure with a language model},
  author={Lin, Zeming and Akin, Halil and Rao, Roshan and Hie, Brian and Zhu, Zhongkai and Lu, Wenting and Smetanin, Nikita and Verkuil, Robert and Kabeli, Ori and Shmueli, Yaniv and others},
  journal={Science},
  volume={379},
  number={6637},
  pages={1123--1130},
  year={2023},
  publisher={American Association for the Advancement of Science}
}

@article{elnaggar2021prottrans,
  title={ProtTrans: toward understanding the language of life through self-supervised learning},
  author={Elnaggar, Ahmed and Heinzinger, Michael and Dallago, Christian and Rehawi, Ghalia and Wang, Yu and Jones, Llion and Gibbs, Tom and Feher, Tamas and Angerer, Christoph and Steinegger, Martin and others},
  journal={IEEE transactions on pattern analysis and machine intelligence},
  volume={44},
  number={10},
  pages={7112--7127},
  year={2021},
  publisher={IEEE}
}

@article{brandes2022proteinbert,
  title={ProteinBERT: a universal deep-learning model of protein sequence and function},
  author={Brandes, Nadav and Ofer, Dan and Peleg, Yam and Rappoport, Nadav and Linial, Michal},
  journal={Bioinformatics},
  volume={38},
  number={8},
  pages={2102--2110},
  year={2022},
  publisher={Oxford University Press}
}

@article{chen2018ifeature,
  title={iFeature: a python package and web server for features extraction and selection from protein and peptide sequences},
  author={Chen, Zhen and Zhao, Pei and Li, Fuyi and Leier, Andr{\'e} and Marquez-Lago, Tatiana T and Wang, Yanan and Webb, Geoffrey I and Smith, A Ian and Daly, Roger J and Chou, Kuo-Chen and others},
  journal={Bioinformatics},
  volume={34},
  number={14},
  pages={2499--2502},
  year={2018},
  publisher={Oxford University Press}
}

@article{chen2020ilearn,
  title={iLearn: an integrated platform and meta-learner for feature engineering, machine-learning analysis and modeling of DNA, RNA and protein sequence data},
  author={Chen, Zhen and Zhao, Pei and Li, Fuyi and Marquez-Lago, Tatiana T and Leier, Andr{\'e} and Revote, Jerico and Zhu, Yan and Powell, David R and Akutsu, Tatsuya and Webb, Geoffrey I and others},
  journal={Briefings in bioinformatics},
  volume={21},
  number={3},
  pages={1047--1057},
  year={2020},
  publisher={Oxford University Press}
}

@article{liu2019bioseq,
  title={BioSeq-Analysis2. 0: an updated platform for analyzing DNA, RNA and protein sequences at sequence level and residue level based on machine learning approaches},
  author={Liu, Bin and Gao, Xin and Zhang, Hanyu},
  journal={Nucleic acids research},
  volume={47},
  number={20},
  pages={e127--e127},
  year={2019},
  publisher={Oxford University Press}
}

@article{lecun2015deep,
  title={Deep learning},
  author={LeCun, Yann and Bengio, Yoshua and Hinton, Geoffrey},
  journal={nature},
  volume={521},
  number={7553},
  pages={436--444},
  year={2015},
  publisher={Nature Publishing Group UK London}
}

@article{jumper2021alphafold2,
  title={Highly accurate protein structure prediction with AlphaFold},
  author={Jumper, John and Evans, Richard and Pritzel, Alexander and Green, Tim and Figurnov, Michael and Ronneberger, Olaf and Tunyasuvunakool, Kathryn and Bates, Russ and {\v{Z}}{\'\i}dek, Augustin and Potapenko, Anna and others},
  journal={nature},
  volume={596},
  number={7873},
  pages={583--589},
  year={2021},
  publisher={Nature Publishing Group UK London}
}

@article{lin2022language,
  title={Language models of protein sequences at the scale of evolution enable accurate structure prediction},
  author={Lin, Zeming and Akin, Halil and Rao, Roshan and Hie, Brian and Zhu, Zhongkai and Lu, Wenting and Smetanin, Nikita and dos Santos Costa, Allan and Fazel-Zarandi, Maryam and Sercu, Tom and Candido, Sal and others},
  journal={bioRxiv},
  year={2022},
  publisher={Cold Spring Harbor Laboratory}
}

@article{ali2022efficient,
  title={Efficient approximate kernel based spike sequence classification},
  author={Ali, Sarwan and Sahoo, Bikram and Khan, Muhammad Asad and Zelikovsky, Alexander and Khan, Imdad Ullah and Patterson, Murray},
  journal={IEEE/ACM Transactions on Computational Biology and Bioinformatics},
  year={2022},
}

@inproceedings{shen2018wasserstein,
  title={Wasserstein distance guided representation learning for domain adaptation},
  author={Shen, Jian and Qu, Yanru and others},
  booktitle={AAAI conference on artificial intelligence},
  year={2018}
}

@inproceedings{xie2016unsupervised,
  title={Unsupervised deep embedding for clustering analysis},
  author={Xie, Junyuan and Girshick, Ross and Farhadi, Ali},
  booktitle={International conference on machine learning},
  pages={478--487},
  year={2016}
}

@article{rao2019evaluating,
  title={Evaluating protein transfer learning with TAPE},
  author={Rao, Roshan and Bhattacharya, Nicholas and Thomas, Neil and Duan, Yan and Chen, Peter and Canny, John and Abbeel, Pieter and Song, Yun},
  journal={Advances in neural information processing systems},
  volume={32},
  year={2019}
}
\end{document}